\documentclass[conference]{IEEEtran}
\IEEEoverridecommandlockouts
\usepackage{amssymb, amsthm, amsmath, color, tikz, natbib}
\usepackage{url}

\usepackage{hyperref}
\newtheorem{theorem}{Theorem}
\newtheorem{lemma}[theorem]{Lemma}

\theoremstyle{remark}
\newtheorem{remark}[theorem]{Remark}
\theoremstyle{definition}
\newtheorem{definition}[theorem]{Definition}
\theoremstyle{definition}

\newcommand{\R}{\mathbb{R}}

\begin{document}
	
\title{Improving Metric Dimensionality Reduction with Distributed Topology 
\thanks{The first author was partially supported by the National Science Foundation under the grant “HDR TRIPODS: Innovations in Data Science: Integrating Stochastic Modeling, Data Representations, and Algorithms”, NSFCCF-1934964. The second and third authors were partially supported by the Air Force Office of Scientific Research under the grant “Geometry and Topology for Data Analysis and Fusion”, AFOSR FA9550-18-1-0266. The MR brain images from healthy volunteers used in this paper were collected and made available by the CASILab at The University of North Carolina at Chapel Hill and were distributed by the MIDAS Data Server at Kitware, Inc. \newline $^*$ Equal contribution}}

\author{
	Alexander Wagner$^*$\\
	Department of Mathematics, \\
	Duke University\\
	Durham, USA \\
	alexander.wagner@duke.edu
	\and
	Elchanan Solomon$^*$\\
	Department of Mathematics, \\
	Duke University\\
	Durham, USA \\
	yitzchak.solomon@duke.edu
	\and
	Paul Bendich\\
	Department of Mathematics, Duke University\\
	Geometric Data Analytics\\
	Durham, USA \\
	paul.bendich@duke.edu
}

\maketitle

\begin{abstract}
We propose a novel approach to dimensionality reduction combining techniques of metric geometry and distributed persistent homology, in the form of a gradient-descent based method called DIPOLE. DIPOLE is a dimensionality-reduction post-processing step that corrects an initial embedding by minimizing a loss functional with both a local, metric term and a global, topological term. By fixing an initial embedding method (we use Isomap), DIPOLE can also be viewed as a full dimensionality-reduction pipeline. This framework is based on the strong theoretical and computational properties of distributed persistent homology and comes with the guarantee of almost sure convergence. We observe that DIPOLE outperforms popular methods like UMAP, t-SNE, and Isomap on a number of popular datasets, both visually and in terms of precise quantitative metrics.    
\end{abstract}

\section{Introduction}
The goal of dimensionality reduction is to replace a high-dimensional data set with a low-dimensional proxy that has a similar shape. The terms \emph{similar} and \emph{shape} are not strictly defined, and a wealth of dimensionality reduction methods exists to accommodate the manifold ways of formulating the problem. Classical methods like PCA and MDS \citep{Kruskal:1964aa} are concerned with globally preserving variance or pairwise distances in a data set. These methods are of great utility in machine learning as a whole, but produce poor results on data sets that are not metrically low-dimensional. This has prompted the introduction of methods that emphasize preserving local structure over global structure, such as locally linear embeddings \citep{Roweis2323}, Laplacian eigenmaps \citep{belkin2003laplacian} and diffusion maps \citep{coifman2006diffusion}, Isomap \citep{Tenenbaum2319}, and more recently, t-SNE \citep{JMLR:v9:vandermaaten08a} and UMAP \citep{McInnes2018}. Each method has its own way of encoding a high-dimensional data set, e.g. via a connectivity graph, using random walks, as a probability distribution, etc., and its own scheme for preserving the metric, spectral, or distributional properties of that data structure in a low-dimensional embedding.

To our mind, there are two important limitations of the above-mentioned dimensionality reduction methods. Firstly, there are many cases in which some or all of the global geometry of a data set can be preserved by a low-dimensional embedding, but these local methods are not guaranteed to find such an embedding. Secondly, there are important global features of shapes that are not metric but topological, and as none of these local methods compute topological invariants, they cannot guarantee preservation of topological structure.

Our goal is to show that distributed persistence, as defined in \cite{2021arXiv210112288S}, can be used to augment local metric methods and address both of the above challenges simultaneously. To be precise, we will show that:
\begin{enumerate}
	\item Distributed persistence provides a scalable, parallelizable framework for incorporating topological losses into dimensionality reduction methods.
	\item Distributed persistence can provide strong guarantees for preservation of global structure, when feasible.
	\item Gradient descent with distributed persistence provably converges.
	\item Incorporating persistence into dimensionality reduction provides improved embeddings on a number of standard data sets.
\end{enumerate} 

We call our method \textbf{DIPOLE}: DIstributed Persistence-Optimized Local Embeddings. The word \emph{dipole} comes from the Greek $\delta \mbox{\'{i}} \zeta$ $\pi \mbox{\'{o}}\lambda o \zeta$ meaning \emph{doubled axes}. In our framework, the two axes are local geometry and global topology, which play both complementary and competitive roles in obtaining optimal embeddings. 

\subsection{Why persistent homology?}
Unlike more ad hoc methods for measuring the shape of data, persistent homology is rooted in algebraic topology and enjoys strong theoretical foundations. Persistent homology captures topology on many scales at once, and so is not limited by the need to fix a scale parameter. The output of persistent homology is differentiable almost everywhere as a function of the input data \citep{gameiro2016continuation} and \citep{poulenard2018topological}, allowing for its incorporation in gradient descent optimization schemes \citep{carriere2021optimizing}. Persistent homological features (barcodes, persistence diagrams, Betti curves, Euler curves, etc.) can be directly compared with one another \citep{cohen2007stability}, subjected to statistical analysis \citep{mileyko2011probability} and \citep{fasy2014confidence}, and transformed into feature vectors \citep{bubenik2015statistical} and \citep{adams2017persistence}.  

\subsection{Why distributed persistence?}
Persistence calculations scale poorly in the size of data sets \citep{otter2017roadmap}, are not robust to outliers \citep{buchet2014topological}, and leave out a lot of information about the shape of data \citep{curry2018fiber}. Distributed persistence overcomes these limitations through random subset sampling: instead of computing the persistence diagram of the full data set, one computes the persistence diagrams of many small subsets. This has the effect of producing a much faster and more robust invariant. Moreover, distributed persistence is considerably more informative than full persistence, as spaces with similar distributed persistences are necessarily quasi-isometric.

It is natural to ask why, if distributed persistence invariants contain global metric data, we do not simply use global metric data directly. There are three main reasons:
\begin{itemize}
	\item When global geometry-preserving embeddings are impossible (which is generally the case), preserving distributed persistence is not equivalent to preserving quasi-isometry type. The latter is futile whereas the former amounts to fixing the topology, which may indeed be possible.
	\item When there are embeddings that preserve global geometry within reasonable error, there may generally be many such embeddings, some of which are more topologically faithful than others. By working with distributed persistence, we can obtain a quasi-isometry that also has nice topological properties.
	\item It is sometimes possible to find an embedding that is a quasi-isometry on a subset of a space. An embedding method based on fixing topology may be able to capture the global geometry on this subset and the topological type of its complement.  
\end{itemize}     

We see DIPOLE as a dimensionality-reduction post-processing step. Starting with an initialization embedding, obtained by methods such as t-SNE, UMAP, or Isomap, DIPOLE corrects the topology using distributed persistence. The use cases of focus in this paper are data sets which exhibit low intrinsic dimension and marked topological and metric structures. Examples include manifolds, algebraic varieties, configuration spaces, graphs, and more general embedded simplicial complexes. For such spaces, it is often possible to find low-dimensional embeddings that preserve geometry on small scales and maintain large-scale topology.

\section{Prior Work}
There have been other lines of research aimed at incorporating topology into dimensionality reduction. \cite{doraiswamy2020topomap} aim to maintain global degree-zero persistence by preserving the minimal spanning tree of the original data embedding in the projection. \cite{shieh2011tree} study embeddings preserving the single-linkage dendogram of a data set. \cite{yan2018homology} consider an extension of landmark Isomap \citep{de2004sparse} that uses homology to pick optimal landmarks. \cite{moor2020topological} add a topological loss term to improve latent space representations of autoencoders. These approaches differ from ours in two ways: only \cite{moor2020topological} optimize on topological invariants directly, and none of the prior methods use distributed persistence.

\subsection{Outline of Paper}
Section \ref{sec:distpers} defines distributed persistence and states theorems from \cite{2021arXiv210112288S} highlighting its inverse properties. Section \ref{sec:dimred} explains the dimensionality reduction pipeline used in DIPOLE. Section \ref{sec:convergence} provides a proof of almost sure convergence for the gradient descent scheme underlying DIPOLE. Section \ref{sec:experiments} compares DIPOLE to other popular methods on a number of challenging data sets. Finally, Section \ref{sec:conclusion} summarizes the results of the paper and outlines directions for future research.

\section{Distributed Persistence}
\label{sec:distpers}
We provide a summary of the results from \cite{2021arXiv210112288S}, and indicate how they relate to the task of dimensionality reduction. Let $\lambda$ be an invariant of finite metric spaces.  Let $X$ be an abstract indexing set, and $\psi:X \to Z$ an embedding into a metric space $Z$. For $k \in \mathbb{N}$, we define the distributed invariant $\lambda_k$ that maps the labeled point cloud $(X,\psi)$ to the labeled set of invariants $\{\lambda(\phi(S)) \mid S \subseteq X, |S| = k\}$ for $k> 0$, and $\emptyset$ otherwise. Put another way, $\lambda_{k}(X,\psi)$ records the values of $\lambda$ on subsets of $\psi(X)$ of a fixed size, together with abstract labels identifying which invariant corresponds to which subset. In what follows, we will omit the embedding map $\psi$ and refer to $X$ as a metric space directly, unless it becomes necessary to distinguish between $X$ as an indexing set and $X$ as a metric space.

When the invariant $\lambda$ is expensive to compute, as is the case for persistent homology, $\lambda_{k}$ provides a scalable, parallelizable alternative. $\lambda_k$ enjoys all of the stability properties of $\lambda$ (cf. \cite{cohen2007stability}), since it is computed in the same way, but is generally more robust to outliers, since $\lambda_k$ is unchanged on those subsets of $X$ that do not contain an outlier. The main result of \cite{2021arXiv210112288S} is that the topology of many small subsets can be used to recover the isometry type of $X$ up to a multiplicative error term that is generally quadratic (but sometimes linear) in the size of the subsets taken:
\begin{theorem}
Let $\lambda^{m}$ be the invariant that associates to a point cloud the Rips persistence of its $m$-skeleton, and take $k > m > 0$. Let $\phi: X \to Y$ be a surjection\footnote{The original version of this result assumes $\phi$ is a bijection, but relaxing it to a surjection has no impact on the proof.} such that for all $S \subseteq X$ with $|S| \in \{k,k-1,\cdots, k-m-1\}$, $d_{B}(\lambda^{m}(S),\lambda^{m}(\phi(S))) \leq \epsilon$.\footnote{The output of $\lambda^{m}$ is a collection of persistence diagrams, one for each homological degree up to $m$. The bottleneck distance here is the maximum of the bottleneck distances across all degrees.} Then $\phi$ is a $112k^2\epsilon$ quasi-isometry, i.e. $\forall x_1, x_2 \in X$,
\[|d_{X}(x_1,x_2) - d_{Y}(\phi(x),\phi(y))| \leq 112k^2 \epsilon.\]
\label{thm:inv}
\end{theorem}
Note that when $\epsilon = 0$, $\phi$ must be an isometry, showing that the invariant sending $X$ to the topology of many small subsets is injective. The above theorem extends injectivity by showing that the inverse is Lipschitz from the Bottleneck distance $d_{B}$ to the quasi-isometry metric. This result can be improved to give linear dependence on the subset size $k$, provided $\phi$ is known to restrict to a quasi-isometry on a small subset of $X$:

\begin{theorem}
	Let $k > m = 1$. Let $\phi: X \to Y$ be a surjection such that for all $S \subseteq X$ with $|S| \in \{k,k-1\}$, $W_{1}(\lambda^{1}(S),\lambda^{1}(\phi(S))) \leq \epsilon_1$.\footnote{The $W_{1}$ distance here is the sum of the $W_{1}$ distances for the zero and first degree persistent homology.} Suppose further that there is a subset $X' \subset X$ of size $(k-1)$ with
	\[\sum_{(x_i,x_j) \in X' \times X'} |d_{X}(x_i,x_j)- d_{Y}(\phi(x_i),\phi(x_j))| \leq \epsilon_2.\]
	Then $\phi$ is a $56(k+1)\epsilon_1 + 28\epsilon_2$ quasi-isometry.
	\label{thm:sparseinv}
	\label{thm:linearthm}
\end{theorem}

One can paraphrase this result as follows: assuming that $\phi$ is already known to preserve  distances on a small subset $X' \subset X$, its maintaining of small-scale topology gives a linear bound on global metric distortion. This motivates a pipeline in which we initialize our embedding $\phi$ to preserve local distances, and add an extra term to our topological loss (the \emph{local metric regularizer})  that protects this local correctness, so that we do not violate the hypothesis of Theorem \ref{thm:linearthm} over the course of the optimization.

\begin{remark}
	\label{remark:sampling}
	In fact, one does not need to show that $\phi$ preserves topology on \emph{all} small subsets of cardinality $|S| \in \{k, \cdots, k-m-1\}$ to ensure a quasi-isometry, as a relatively small number of subsets suffices. Details and extensions of these results can be found in \cite{2021arXiv210112288S}.
\end{remark}

\section{Dimensionality Reduction}
\label{sec:dimred}
Having provided an overview of the definition and advantages of distributed persistence, we now go into the details of our pipeline. Let $(X,d_X)$ be a finite metric space, not necessarily embedded in Euclidean space. For a given $n \in \mathbb{N}$, our goal is to find a function $f: X \to \mathbb{R}^n$ that preserves important metric and topological features of $X$. To that end, we want to define a loss on embedding functions that is sufficiently differentiable to allow for gradient descent optimization. We define a loss with both a local metric term and a global topological term.

We begin with the local metric term. Let $P \subseteq X \times X$ be a subset of pairs of points in $X$. Intuitively, the pairs of points in $P$ should be thought of as being close to one another. One way of producing such a set $P$ is to pick a radius parameter $\delta > 0$ and define:
\[P= \{(x_i,x_j) \in X \times X \mid d_{X}(x_i,x_j) \leq \delta \}.  \]  
Another possibility is to pick an integer $\ell$ and define $P$ to consist of those pairs $(x_i,x_j)$ where $x_i$ is one of $x_j$'s $\ell$-nearest neighbors, or the reverse. 

With our set $P$ in hand, we can define the \emph{local metric regularizer functional}, $LMR_{P}$, defined as follows:
\[LMR_{P}(f) = \sum_{(x_1,x_2) \in P} \left( d_{X}(x_1,x_2) - \|f(x_1) - f(x_2)\| \right)^{2}. \]

For the global topological term of our loss, we turn to distributed persistence, and define the functional $DP_{k}^{m}(f)$ as follows:
\[DP_{k}^{m}(f) = \frac{1}{{|X| \choose k}}\sum_{\substack{S \subset X\\ |S| = k}} W_{p}^{p}(\lambda^{m}(S), \lambda^{m}(f(S))).  \]

where $W_{p}$ is the $p$-Wasserstein distance between diagrams. Since we are working with persistent homology in multiple degrees, we sum up the  $W_{p}^{p}$-distance among all degrees. Viewing the embedding $f$ as a matrix of shape $n \times |X|$, this functional is differentiable almost everywhere with respect to the entries of the matrix, so that gradient descent optimization can be applied. Due to the huge number of terms in the sum, and in light of Remark \ref{remark:sampling}, we opt for stochastic gradient descent, randomly sampling a small batch of subsets of size $k$ whose persistence are computed. Thus our pipeline is the following: we initialize our embedding $f_0$ and then apply some form of gradient descent to the functional $\alpha LMR_{P} + DP_{k}^{m}$. 

\section{Convergence}
\label{sec:convergence}
Our aim in this section is to prove that gradient descent on the functional $\Phi = \alpha LMR_{P} + DP_{k}^{m}$ converges almost surely. Assuming our point cloud consists of $s$ points in $\mathbb{R}^d$, we can encode it as an element of $\mathbb{R}^{ds}$. We write $X_0 \in \mathbb{R}^{ds}$ for the initial point cloud, and $X_n \in \mathbb{R}^{ds}$ for the point cloud after $n$ stochastic gradient descent steps. Our proof relies on the techniques developed in \cite{davis2020stochastic}, which we now recall.

Let $\Phi :\mathbb{R}^{ds} \to \mathbb{R}$ be locally Lipschitz, and let $\partial \Phi(x)$ denote the \emph{Clarke subdifferential} of $\Phi$ at a point $X \in \mathbb{R}^{ds}$. Starting at an initial point $X_0 \in \mathbb{R}^{ds}$, the \emph{stochastic subgradient method} consists of the following iteration: $X_{n+1} = X_{n} - \alpha_n (y_n + \xi_n)$, where $\alpha_n$ is a step size, $y_n \in \partial \Phi(X_n)$ is an element of the Clarke subdifferential, and $\xi_n$ is a sequence of random variables modeling noise. It is shown in \cite{davis2020stochastic} (Corollary 5.9) that the following four conditions (the first three are called \emph{Assumption C}) guarantee almost sure convergence:
\begin{enumerate}
	\item The sequence $\alpha_n$ is nonnegative, square summable ($\sum_{n} \alpha_n^2 < \infty$) but not summable ($\sum_{n} \alpha_n = \infty)$.
	\item The iterates are a.s. bounded, $\sup \|X_n\| < \infty$.
	\item The sequence $\{\xi_n\}$ is a martingale difference sequence, i.e. almost surely we have $\mathbb{E}[\xi_n] = 0$ and $\mathbb{E}[\|\xi_n\|^2] < p(x_n)$, where $p$ is some function bounded on bounded sets.
	\item $f$ is $C^d$-stratifiable.
\end{enumerate}

For definitions of technical terms like Clarke subdifferential or stratifiability, cf. \cite{davis2020stochastic}. 

\subsection{Our Setting}
Before proceeding to the proof, we explain how DIPOLE fits into the framework of \cite{davis2020stochastic}.

We perform gradient descent by randomly picking some number of subsets of size $k$, and using the gradients on those subsets to take a descent step. Since $\partial \Phi$ takes into account \emph{all} subsets, we need some way of subtracting off the gradients for all those subsets not sampled. This is accomplished by a careful definition of $\xi_n$. To begin, we randomly pick $b$ subsets $S_1, \cdots, S_b$ of size $k$, consider the sum $\hat{\Phi} = \frac{1}{b} \sum_{i=1}^{b} W_{p}^{p}(\lambda^{m}(S_i),\lambda^{m}(\phi(S_i))) + \alpha LMR_{P}$, and take a random element $v_n \in \partial \hat{\Phi}(x_n)$. It is this element $v_n$ that we use in DIPOLE. We therefore set $\xi_n = y_n - v_n$, so that $y_n - \xi_n = v_n$. In other words, we define $\xi_n$ to cancel out the effect of all unsampled subsets, leaving only the sampled gradient.

We now point out two important features of $\Phi$. Firstly, $\Phi$ is the sum of a smooth function (the local metric regularizer) and persistence functionals (which are locally Lipschitz), so $\Phi$ itself is locally Lipschitz, and the results of \cite{davis2020stochastic} apply. Secondly, since both the local metric regularizer and persistence functionals are translation invariant, $\Phi$ is also translation invariant. We use the translation invariance to replace the original sequence $X_n$ with a mean-centered sequence $\bar{X}_n$ that is easier to control.

\begin{definition}
	Let $X \in \mathbb{R}^{ds}$ represent a point cloud of $s$ points in $\mathbb{R}^d$. Define $\bar{X} \in \mathbb{R}^{ds}$ to be the result of mean-centering the point cloud. Identifying $\mathbf{R}^{ds}$ with the vector space of matrices $M_{d \times s}$, we have
	\[\bar{X} = X(I - \frac{1}{s}J) \]
	where $I$ is the $s \times s$ identity matrix, and $J$ is an $s \times s$ matrix containing only $1$s.
\end{definition}

The following theorem requires that $(X,P)$ be a connected subgraph of $X$. This is analogous to the assumption inherent in Isomap that the shortest path metric on the $k$-nearest-neighbors subgraph of a point cloud has no infinite values.

\begin{theorem}
Suppose $\alpha > 0$ and $(X,P)$ is a connected subgraph of $X$. Then the mean-centered sequence $\bar{X}_n$ converges to a critical point for $\Phi$, and $\Phi(X_n) = \Phi(\bar{X}_n)$ converges.
\end{theorem} 

\subsection{The Proof}

Before working through the four conditions guaranteeing a.s. convergence in \cite{davis2020stochastic}, we justify replacing $X_n$ with $\bar{X}_n$. In order to apply the convergence results of \cite{davis2020stochastic} to $\bar{X}_n$, we need to show that $\bar{X}_n$ also arises from a stochastic gradient descent scheme. The subtlety here is the following: $\bar{X}_n$ and $\bar{X}_{n+1}$ are obtained by mean-centering $X_n$ and $X_{n+1}$, respectively, but it is not necessarily the case that a gradient descent step at $\bar{X}_{n}$ produces $\bar{X}_{n+1}$, i.e. that mean-centering commutes with taking gradients. A failure of commutativity would mean that $\{\bar{X}_n\}$ is not the product of an actual stochastic gradient descent scheme, and hence is not eligible for the application of prior convergence results.

To show commutativity, we observe that $\bar{X}$ sits inside the subvariety $\mathcal{X}$ of $M_{d \times s}$ consisting of mean-centered point clouds. $\mathcal{X}$ can also be characterized as the range of the projection $T(X) = X(I - \frac{1}{s}J)$. As $I - \frac{1}{s}J$ is symmetric, $T$ is an orthogonal projection. Given a point $X \in \mathcal{X}$ and a gradient $\epsilon \in \partial \Phi(X)$ in the full tangent space to $M_{d \times s}$, we then know that the orthogonal projection of $\epsilon$ to the tangent space at $\mathcal{X}$ is $T(\epsilon)$. Thus, if $X_{n+1} = X_{n} + \alpha_n \epsilon$,
\begin{align*}
	\bar{X}_{n+1} & = T(X_{n+1})\\
	& = T(X_{n} + \alpha_n \epsilon)\\
	& = T(X_n) + T(\alpha_n \epsilon)\\
	& = \bar{X}_{n} + \alpha_n T(\epsilon).
\end{align*}

This demonstrates that $\bar{X}_{n+1}$ is obtained from $\bar{X}_n$ by taking a descent step in $\mathcal{X}$, and so mean-centering commutes with gradient descent. Lastly, since $\mathcal{X}$ is isometric to $\mathbb{R}^{(d-1) \times s}$, the results of \cite{davis2020stochastic}, framed in Euclidean space, apply equally well to $\mathcal{X}$. The major advantage of working with $\{\bar{X}_n\}$ is the following lemma.
\begin{lemma}
The sequence $\{\bar{X}_n\}$ is a.s. bounded. 
\end{lemma}
\begin{proof}
We first show that the diameter of the mean-centered point clouds $\{\bar{X}_n\}$ is a.s. bounded. Note that the diameter of a point cloud can be bounded without the point cloud being bounded; for a simple example, consider a single point (diameter zero) wandering off to infinity.

Suppose that the diameter of our point cloud is not bounded. Since $(X,P)$ is connected, some edge in $P$ is growing without bound, and hence the local metric regularizer is growing without bound. Clearly, the gradient term coming from the local metric regularizer cannot cause this to happen, so the only possibility is that this is an effect of the topological loss. However, a topological gradient on a very long edge corresponds to a point in a persistence diagram with very large birth or death time, and since the target distributed persistence is bounded, a Wasserstein matching will either push such a point closer to a bounded region in the upper-half quadrant, or else its projection on to the diagonal, and in either case it will not grow without bound.

To conclude that the point clouds $\{\bar{X}_n\}$ are bounded, suppose again the converse. This means that some coordinate of some point is growing without bound. In order for the point cloud to remain mean-centered, the same coordinate of another point must also be growing without bound, and of the opposite sign. This forces the diameter of the point cloud to grow without bound, which we have already shown is impossible. Thus, the mean-centered sequence is a.s. bounded.
\end{proof}

 We now work through the four conditions guaranteeing a.s. convergence, noting that mean-centering has no effect on local quantities like the errors $\xi_n$.

(1) The condition on the learning rate $\alpha_n$ is easily satisfied by setting $\alpha_n = \frac{1}{n}$, as well as any other of a number of annealing schemes. 

(2) This was shown in the Lemma above.

(3) To see that $\mathbb{E}[\xi_n] = 0$, observe that any subset of $X$ is equally likely to be chosen when generating $\xi_n$, so the expected value of $v_n$ is an element of the full Clarke subdifferential $\partial \Phi(X_n)$. Since $X_n$ is generically non-critical, $\partial \Phi(X_n)$ consists of a single vector, so that $\mathbb{E}[v_n] = y_n$ and $\mathbb{E}[\xi_n] = 0$. To see that the variance of $\xi_n$ is bounded, note that $\Phi$ is the sum of finitely many functions, each of which has subgradient bounded when the point cloud has bounded diameter. Thus the subgradients of $\Phi$ and $\hat{\Phi}$ are likewise bounded on bounded sets, and so the same is true for $\|\xi_n\|^2$, even without having to take expectations.

(4) The last condition, that $\Phi$ is $C^d$-stratifiable, follows from persistence functionals being locally linear (details can be found in \cite{carriere2021optimizing}).

\section{Experiments}
\label{sec:experiments}
The exact formula for the loss function used in experiments is given in Equation \eqref{eqn:dipole}. Like other dimensionality reduction pipelines, DIPOLE depends on the choice of hyperparameters. Our default choices are:
\begin{itemize}
\item The default initialization embedding is produced via Isomap. Isomap is fast and non-stochastic, and is designed to preserve local distances, a criterion in line with Theorem \ref{thm:sparseinv}.
\item When the input data is a point cloud in Euclidean space, it is transformed into a distance matrix by taking the geodesic distance on the $m_1$-nearest-neighbor graph, for $2 \leq m_1 \leq 10$.
\item We set the learning rate in $\{0.1,1\}$, and the tradeoff parameter $\alpha \in \{0.01,0.1\}$.
\item The set $P$ is chosen to consist of the $m_2$-nearest neighbors of points in $X$. We typically take $2 \leq m_2 \leq 5$.
\item We fix $p = 2$ in the Wasserstein distance.
\item We run $2500$ descent steps, using the annealing factor $\frac{1000}{1000 + \mbox{step}}$.
\end{itemize} 

The target dimension and subset size parameter $k$ are not fixed in advance, and are considered user-defined parameters.\\

In subsection \ref{sec:visualization} we introduce some classic datasets and visually compare DIPOLE with t-SNE or UMAP. In subsection \ref{subsec:quant} we conduct a quantitative comparison for those same datasets, across a host of parameter values and quality measures, to get a more comprehensive picture of how DIPOLE compares with other, well-known methods. \footnote{DIPOLE source code and scripts to run the experiments can be found at \url{https://github.com/aywagner/DIPOLE}}

\subsection{Visualization Experiments}
\label{sec:visualization}
In this section, we apply DIPOLE to a number of datasets, and visually compare the results with other dimensionality-reduction methods. The datasets we consider are:
\begin{itemize}
	\item A subsample of the mammoth data set from \cite{mammoth}. Cf. Figure \ref{fig:mammoth}.
	\item A point cloud sampled from a brain artery tree, taken from \cite{BULLITT20051232} and analyzed in \cite{bendich2016persistent}. Cf. Figure \ref{fig:brain}.
	\item A swiss roll data set of $3000$ points with the interior of the cylinder $x^2 + (y-1)^2 = 25$ subsequently removed. Cf. Figure \ref{fig:swisshole}.
	\item The \emph{Stanford faces} dataset, consisting of $698$  synthetic face images of  $64 \times 64$ resolution, considered as a point cloud in $\mathbb{R}^{4096}$. The faces were captured as 2D images for a variety of pose and lighting parameters. Cf. Figure \ref{fig:stanfordfaces}.
\end{itemize}

For each experiment, we compare the initial point cloud (or its Isomap visualization), DIPOLE, t-SNE or UMAP, and a degenerate version of DIPOLE where only the local metric regularizer is active (corresponding to $\alpha = 1.0$ in Equation \eqref{eqn:dipole}), denoted \emph{local metric regularizer} (\emph{lmr} for short). For the local metric regularizer experiments, we increase $m_2$ to $10$ to improve its performance, giving local geometry a better chance of enforcing global structure. We see in many cases that local metric regularization alone does not suffice to produce an optimal embedding, and that the differences between its embedding and that obtained by DIPOLE are notably topological, in the relative configurations of distinct components or the appearance of important cycles.

\begin{center}
	\begin{figure}[htb!]
\includegraphics[width=0.5\textwidth]{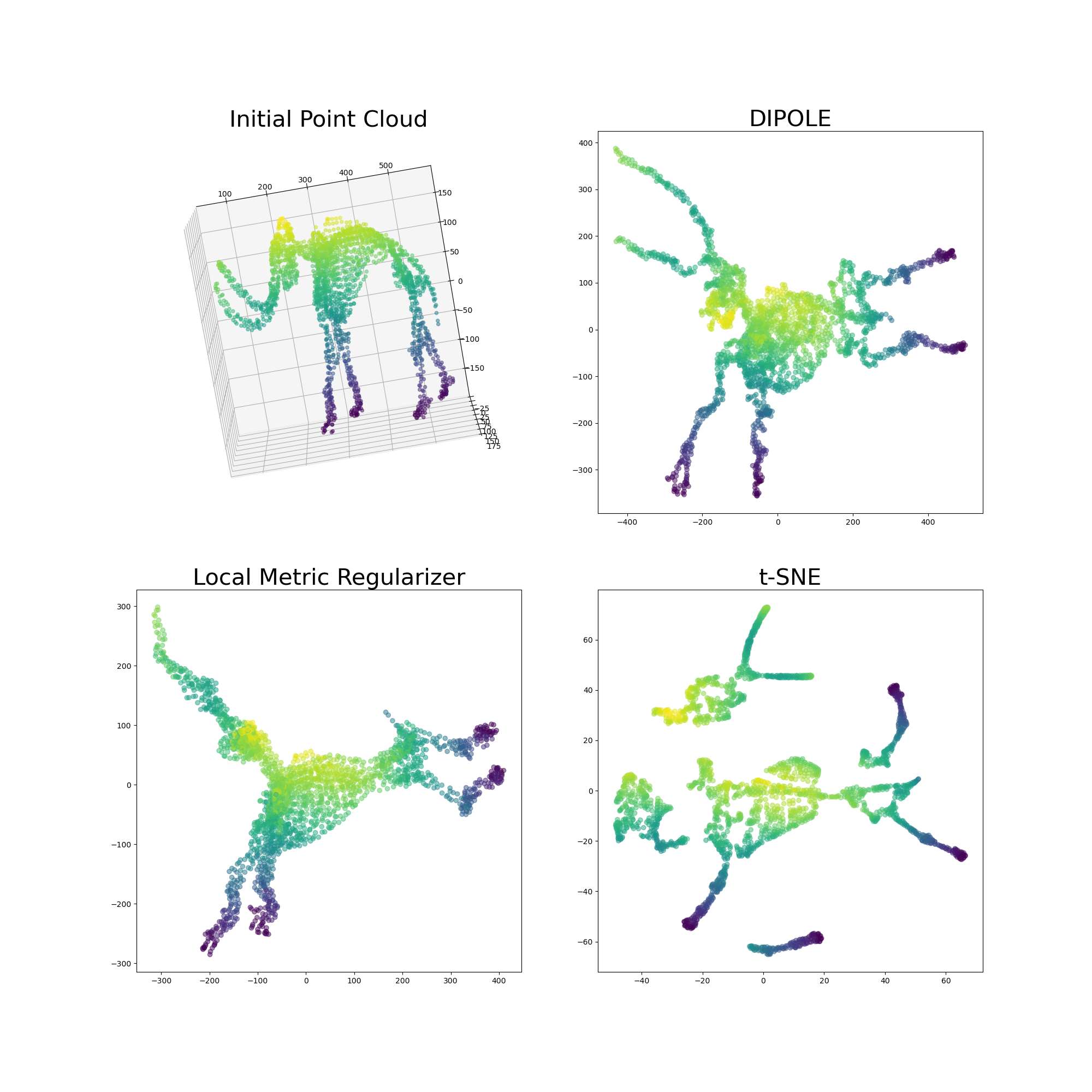}
\caption{Mammoth data set. Top Left: Initial point cloud. Top Right: DIPOLE with $m_1 =5 , m_2 = 3, \alpha = 0.1, k = 64, lr = 1.0$. Bottom Left: Setting $m_2 = 10$ and $\alpha = 1.0$, i.e. using only the local metric regularizer. Bottom Right: t-SNE default parameters.}
\label{fig:mammoth}
\end{figure}
\end{center}

\begin{center}
	\begin{figure}[htb!]
		\includegraphics[width=0.5\textwidth]{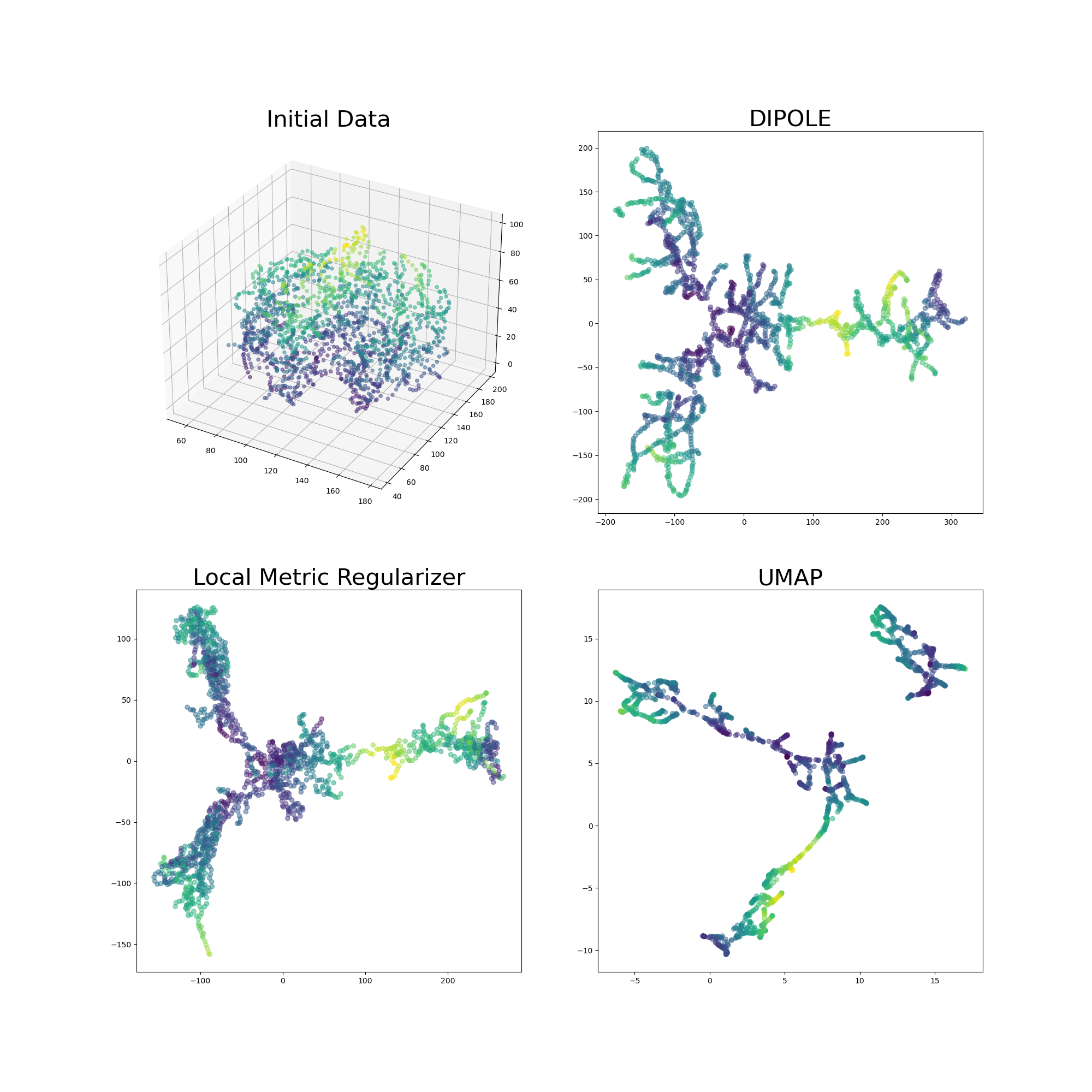}
		\caption{Brain Artery Tree data set. Top Left: Initial point cloud. Top Right: DIPOLE with $m_1 =5 , m_2 = 3, \alpha = 0.1, k = 64, lr = 1.0$. Bottom Left: Setting $m_2 = 10$ and $\alpha = 1.0$, i.e. using only the local metric regularizer. Bottom Right: UMAP default parameters.}
		\label{fig:brain}
	\end{figure}
\end{center}

\begin{center}
	\begin{figure}[htb!]
		\includegraphics[width=0.5\textwidth]{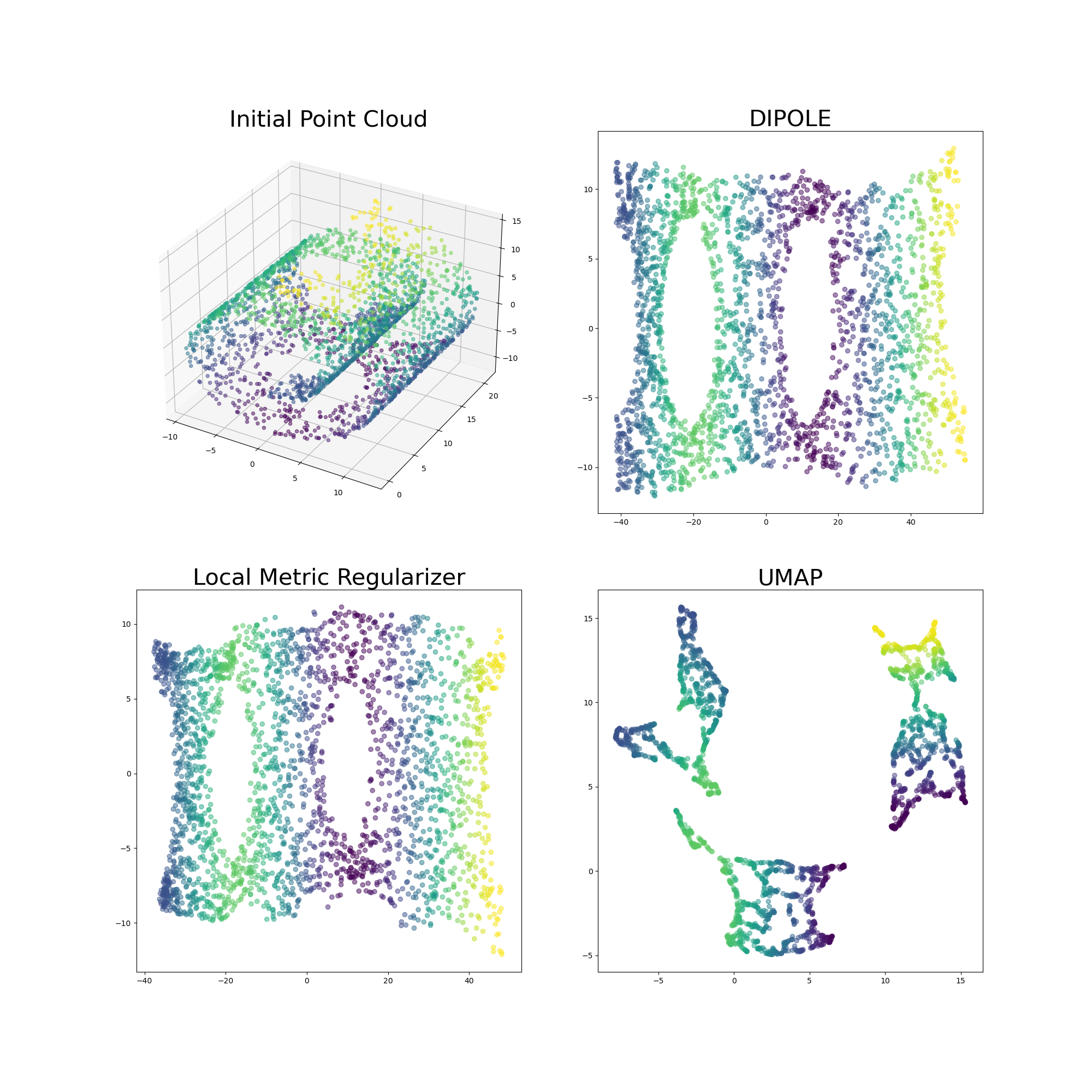}
		\caption{Swiss Roll with Holes data set. Top Left: Initial point cloud. Top Right: DIPOLE with $m_1 =10 , m_2 = 3, \alpha = 0.1, k = 64, lr = 0.1$. Bottom Left: Setting $m_2 = 10$ and $\alpha = 1.0$, i.e. using only the local metric regularizer. Bottom Right: UMAP default parameters.}
		\label{fig:swisshole}
	\end{figure}
\end{center}

\begin{center}
	\begin{figure}[htb!]
		\includegraphics[width=0.5\textwidth]{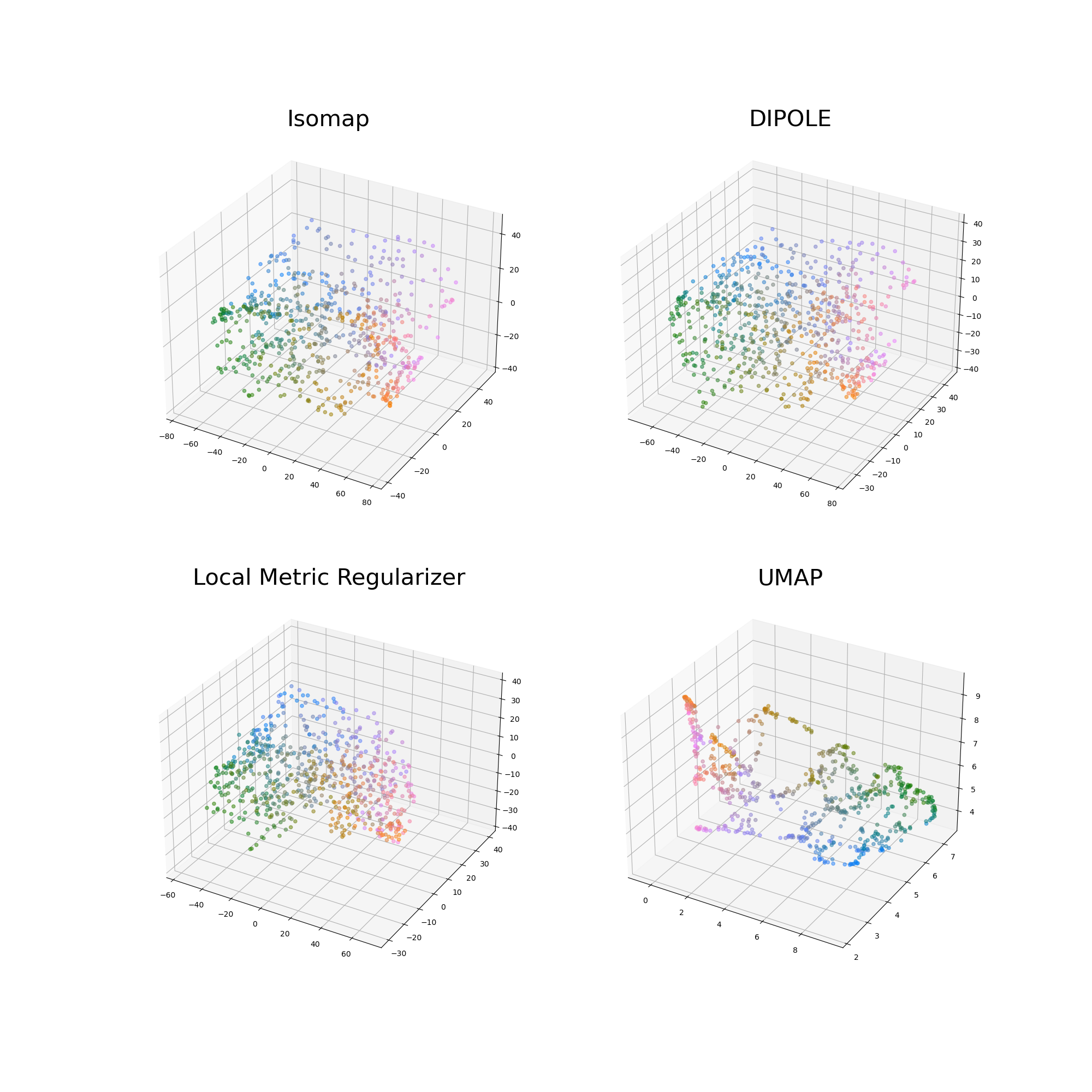}
		\caption{Stanford Faces data set. Top Left: Isomap embedding. Top Right: DIPOLE with $m_1 =5 , m_2 = 5, \alpha = 0.1, k = 32, lr = 1.0$. Bottom Left: Setting $m_2 = 10$ and $\alpha = 1.0$, i.e. using only the local metric regularizer. Bottom Right: UMAP default parameters.  The coloring scheme is RGB, with the pose parameter giving the red scale, and the lighting parameter the blue scale (green was fixed at $0.5$).}
		\label{fig:stanfordfaces}
	\end{figure}
\end{center}

\subsection{Quantitative Analysis}
\label{subsec:quant}
We now perform a quantitative comparison of DIPOLE with other dimensionality reduction methods. Quantitative comparison of different dimensionality reduction techniques is an active area of research without a standard test \citep{wang2020understanding}. The framework we adopt assumes that a dimensionality reduction method has access to three inputs: a point cloud $X \subseteq \R^D$, a target dimension $d < D$ in which to embed $X$, and a specification of an intrinsic metric on $X$, e.g. the geodesic metric on the graph whose edges correspond to the $k$ nearest neighbors of each point in $(X, \| \cdot \|)$. We then compare dimensionality reduction techniques by their capacity to preserve this given intrinsic metric. That is, if $(X, d_H)$ is the user-specified intrinsic metric on $X$, and $(X, d_L)$ is the metric on $X$ given by pulling back the Euclidean metric of its projection to $\mathbb{R}^d$, we are interested in measuring how well $(X, d_L)$ maintained the structure of $(X, d_H)$ according to various tests. It should be noted that all the tests are defined such that lower values represent superior performance. 

The first test we consider is the $ijk$-test. This test measures the probability that the numerical order of $d_H(x_i, x_j), d_H(x_i, x_k)$ is preserved in $d_L$. Precisely, consider the random variable $Z$ defined as follows. Randomly select $x_i, x_j, x_k$ uniformly from the set $X$. If either $d_H(x_i, x_j) \leq d_H(x_i, x_k)$ and $d_L(x_i, x_j) \leq d_L(x_i, x_k)$ or $d_H(x_i, x_j) \geq d_H(x_i, x_k)$ and $d_L(x_i, x_j) \geq d_L(x_i, x_k)$ then $Z$ is $1$; otherwise, $Z$ is $0$. The measure of interest is then $1 - E[Z]$, which we estimate empirically with $10000$ samples.

The second test is the residual variance. Specifically, this is $1 - R^2$ where $R$ is the Pearson correlation coefficient between the distance matrices corresponding to $d_H$ and $d_L$, i.e. it is the cosine of the angle between $d_H$ and $d_L$ viewed as mean-centered vectors. This test was used in \cite{Tenenbaum2319} to quantitatively evaluate the success of Isomap.

The third and fourth tests are the $2$-Wasserstein distance between approximations of the global persistence diagrams in degree $0$ and $1$ of $d_H$ and $d_L$. Specifically, we perform farthest point sampling on $d_H$ and $d_L$ to obtain two subsets of size $256$, compute the Rips persistence of the two metric subspaces, and finally report the $2$-Wasserstein distance between the persistence diagrams.

The five dimensionality reduction techniques that we compare are Isomap, t-SNE, UMAP, DIPOLE, and DIPOLE with only the local metric regularizer. For each method, all combinations of the listed hyperparameters are computed. For DIPOLE, we initialize with Isomap and minimize the following 
\begin{equation}
	\begin{split}
\label{eqn:dipole}
\frac{1-\alpha}{2} \binom{|X|}{k}^{-1}\sum_{\substack{S \subset X\\ |S| = k}} \sum_{p=0}^1 W_2^2(D_p(S, d_H), D_p(S, d_L)) + \\
\alpha \sum_{(i, j) \in P} (d_X(x_i, x_j) - \|a_{i} - a_{j}\|)^2,
	\end{split}
\end{equation}
where $D_p(S, d)$ is the degree $p$ persistence diagram of the Rips complex of the metric space $(S, d)$.

 For DIPOLE, the edges for the local metric regularizer are chosen to be either the $3$ or $5$ nearest neighbors in the Euclidean distance in $\R^D$. The subset size $k$ is chosen to be either $32$ or $64$. The learning rate and $\alpha$ are chosen to be in $\{0.1, 1, 2\}$ and $\{0.01, 0.1\}$, respectively. We do not tune any hyperparameters for Isomap, using the default value of $5$ for the number of nearest neighbors. For the local metric regularizer, we choose the number of edges in the neighborhood graph to be in $\{2, 3, 5, 10\}$ and the learning rate to be in $\{0.1, 1, 2\}$. For t-SNE, we choose $32$ evenly spaced values between $2$ and $75$ for the perplexity hyperparameter. Finally, for UMAP, we choose the number of neighbors in $\{4, 8, 16, 32, 64, 128\}$ and the minimum distance parameter in $\{0, 0.1, 0.25, 0.5, 0.8\}$.

\begin{figure}
	\centering
	\includegraphics[width=0.5\textwidth]{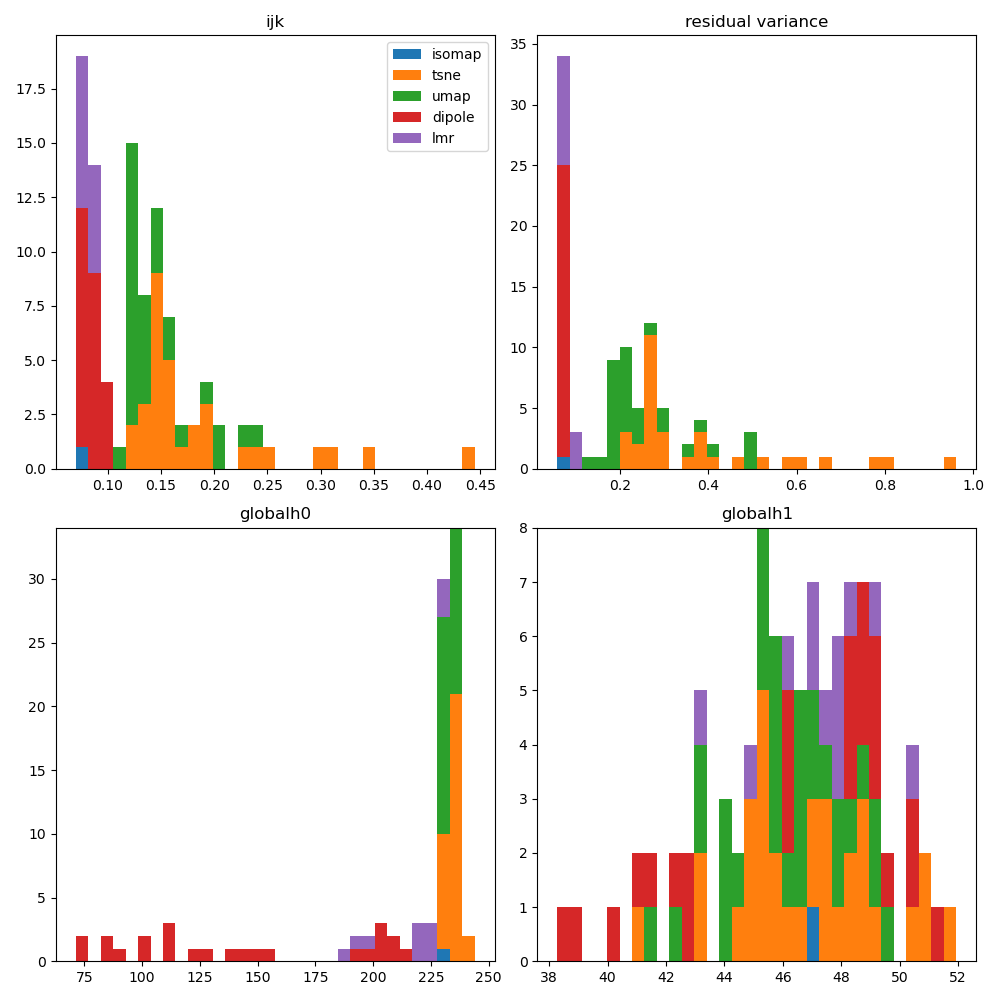}
	\caption{Results for the mammoth dataset.}
	\label{fig:mammothmetrics}
\end{figure}

\begin{figure}
	\centering
	\includegraphics[width=0.5\textwidth]{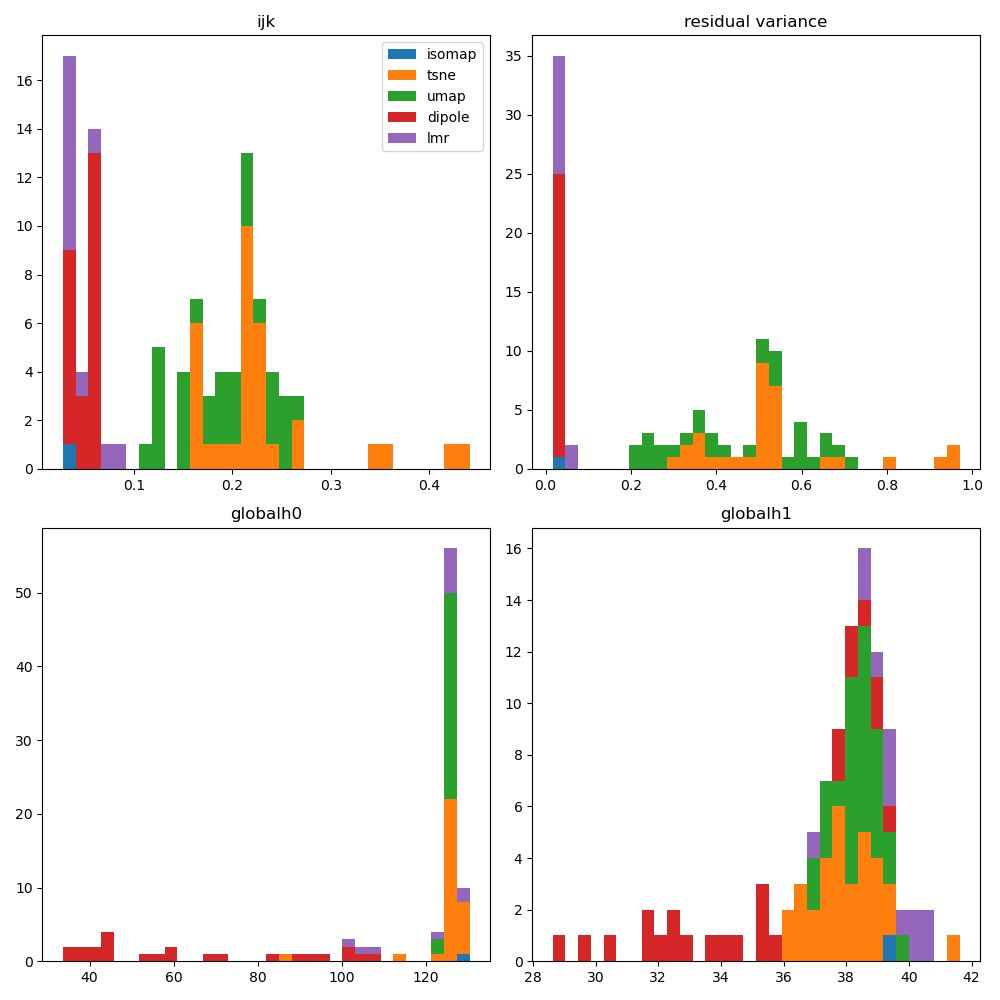}
	\caption{Results for the brain dataset.}
	\label{fig:brainmetrics}
\end{figure}

\begin{figure}
	\centering
	\includegraphics[width=0.5\textwidth]{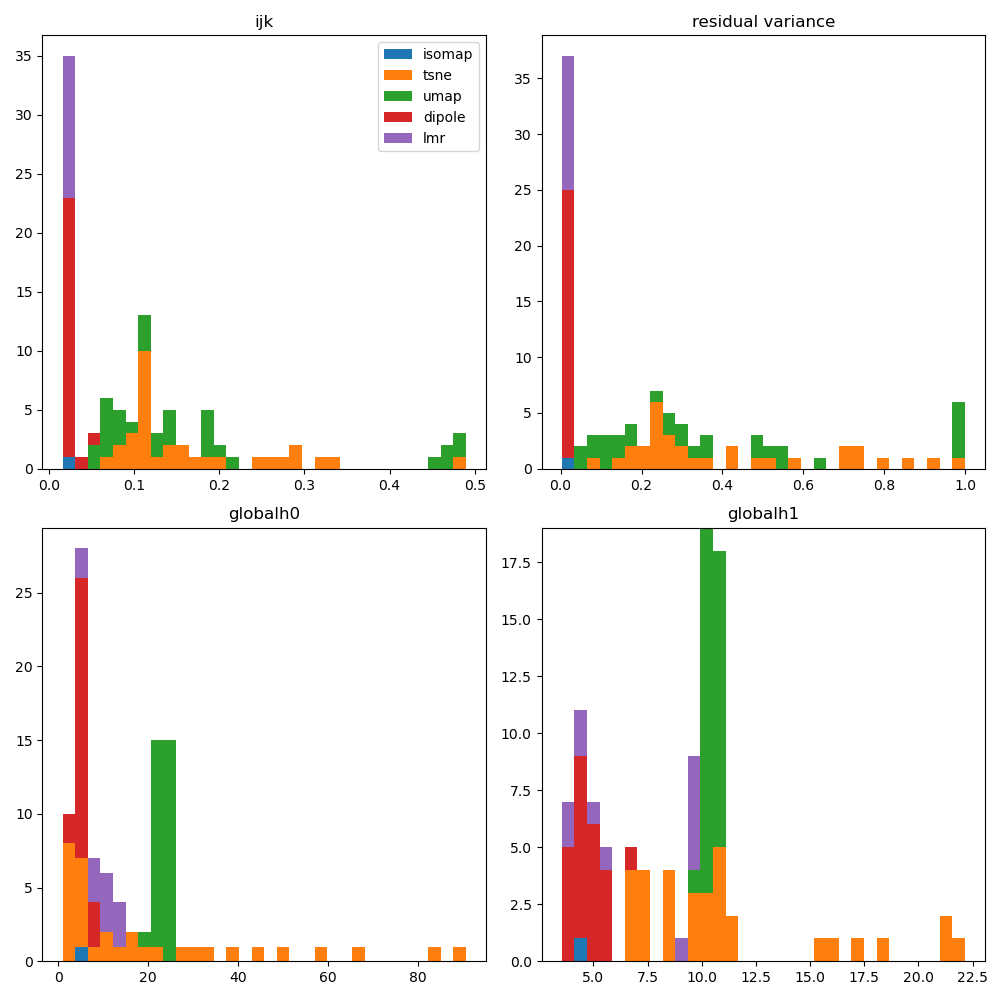}
	\caption{Results for the swiss roll with a hole dataset.}
	\label{fig:swissholemetrics}
\end{figure}

\begin{figure}
	\centering
	\includegraphics[width=0.5\textwidth]{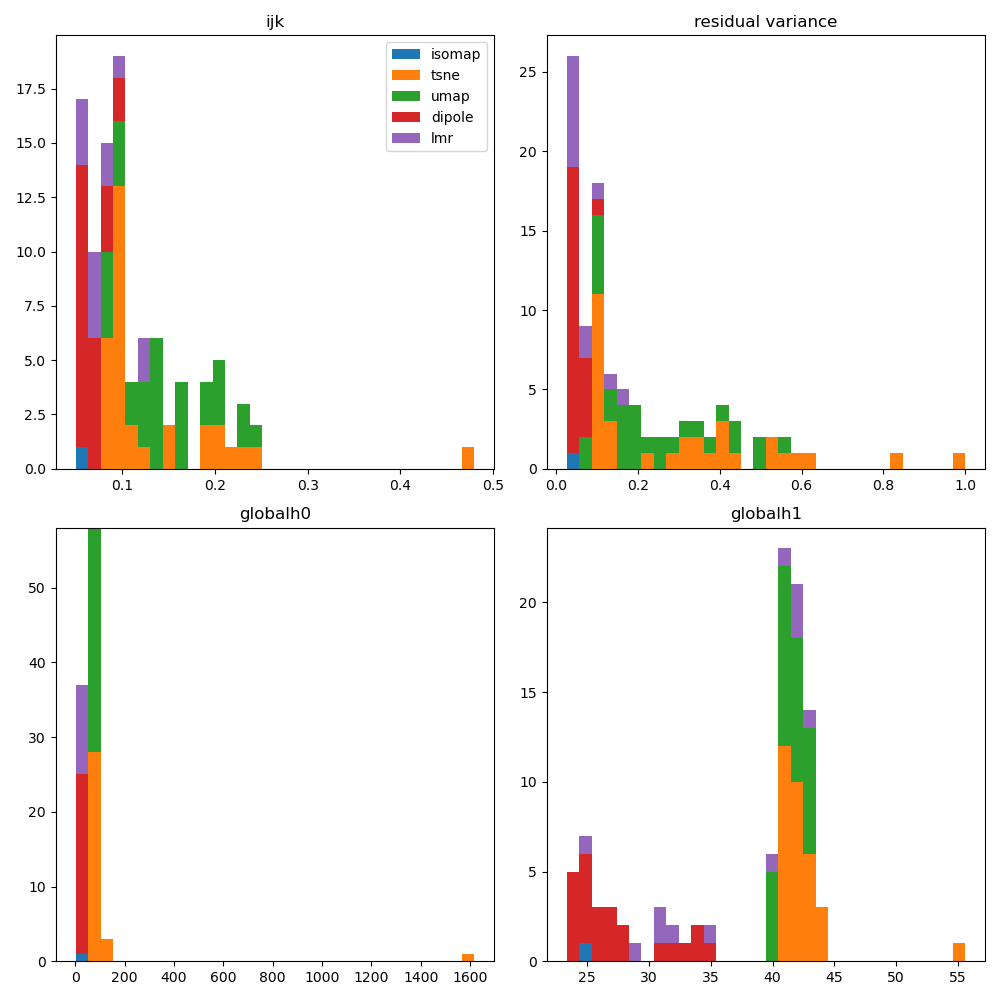}
	\caption{Results for the \emph{Stanford faces} dataset.}
	\label{fig:stanfordfacesmetrics}
\end{figure}

The results for the mammoth, brain, swiss roll with holes, and \emph{Stanford faces} datasets are shown in Figure \ref{fig:mammothmetrics}, \ref{fig:brainmetrics}, \ref{fig:swissholemetrics}, and \ref{fig:stanfordfacesmetrics}. Each figure has four panels corresponding to each of the four tests: $ijk$, residual variance, global degree $0$ persistent homology, and global degree $1$ persistent homology. In each panel, a histogram of scores is plotted for each method over all choices of hyperparameters. Since different methods use hyperparameter grids of different sizes, the histograms for some methods contain more points than the histograms for others; in particular, Isomap only has a single bar of height one in each figure because there was no hyperparameter tuning. 

In all the experiments, most of the hyperparameter choices for DIPOLE result in competitive embeddings with respect to the $ijk$ and residual variance tests. Consistent with the images in Section \ref{sec:visualization}, DIPOLE significantly outperforms all other methods in the degree $0$ and $1$ homology tests for the mammoth and brain datasets.

 One limitation of our scheme for visualizing quantitative scores is that the distribution of scores for each test does not capture their full joint distribution, in that hyperparameters which perform well for one test might not perform well for another test. Readers interested in comparing quantitative scores at this level of granularity are invited to access the publicly available code, from which full tables of scores can be accessed.

 Finally, in Figure \ref{tbl:speed}, we give the average running time of each method and dataset over the hyperparameter grid. Two factors which contribute to the longer runtime of DIPOLE are the time spent computing an initial embedding, done in our experiments with Isomap, and the time spent computing persistence diagrams of size $k$ subsets. In our code, persistence diagrams were computed sequentially, but computing multiple persistence diagrams in parallel, i.e. using a batch size larger than $1$, would result in a substantial decrease in the runtime of DIPOLE. Indeed, this parallelizability was a motivating consideration in developing distributed persistence and will be a feature of future iterations of DIPOLE.
 \begin{figure}
 \begin{center}
	\begin{tabular}{ |c|c|c|c|c|c| } 
	 \hline
	 Dataset & Isomap & t-SNE & UMAP & DIPOLE & LMR \\
	 \hline 
	 Mammoth & 8.50 & 23.86 & 13.68 & 79.10 & 20.17 \\
	 \hline
	 Brains & 6.09 & 17.93 & 10.63 & 74.29 & 15.30 \\
	 \hline
	 Swiss hole & 9.71 & 27.56 & 16.10 & 80.68 & 24.14 \\
	 \hline
	 Faces & 6.74 & 15.10 & 8.56 & 75.01 & 16.78\\
	 \hline
	\end{tabular}
	\caption{Average running time in seconds of each dimensionality reduction method and dataset over hyperparameter grid. Note that the computation of persistence diagrams of subsets is parallelizable and, though not implemented here, would result in a significant increase in the speed of DIPOLE.}
	\label{tbl:speed}
\end{center}
\end{figure}

\section{Conclusion}
\label{sec:conclusion}
Distributed persistence summarizes the multi-scale topology of a data set in a way that is scalable, robust, and differentiable. Combining distributed persistence with local metric information, {\bf DIPOLE} forms a new approach to dimensionality reduction based on structured, topological invariants, and with strong theoretical guarantees.

The pipeline proposed in this paper is very flexible, and open to various modifications and improvements. These include different schemes for sampling subsets (e.g. random point neighborhoods for studying mesoscale phenomena in materials science), other topology-based metrics, alternative persistence constructions (e.g. witness complexes), etc. 

Possible directions for future research include building a form of {\bf DIPOLE} for shape registration and interpolation, as well as incorporating distributed persistent homology into machine learning pipelines for regression and classification.

\end{document}